\documentclass[letterpaper, 10pt, conference]{ieeeconf}

\IEEEoverridecommandlockouts
\usepackage{amsmath,amssymb}

\usepackage{amsthm}
\usepackage{graphicx}
\usepackage{cite}
\usepackage{booktabs}
\usepackage[font=small,labelfont=bf]{caption}
\usepackage{subcaption}
\newtheorem{theorem}{Theorem}
\newtheorem{proposition}{Proposition}
\newtheorem{corollary}{Corollary}
\newtheorem{lemma}{Lemma}
\newtheorem{remark}{Remark}
\newtheorem{assumption}{Assumption}
\newtheorem{definition}{Definition}

\title{\LARGE \bf
Mismatch-Aware Adaptive Constraint Tightening\\
for Bicycle-Model Trajectory Optimization}

\author{Lingxue Lyu$^{1}$ \and Zihui Liu$^{2}$%
\thanks{$^{1}$Lingxue Lyu is with the School of Engineering and Applied
Science, University of Pennsylvania, Philadelphia, PA 19104, USA
{\tt\small lingxuelyu@alumni.upenn.edu}}%
\thanks{$^{2}$Zihui Liu is with the Department of Aeronautics \&
Astronautics, Stanford University, Stanford, CA 94305, USA
{\tt\small zl90@alumni.stanford.edu}}%
}

\begin{document}

\maketitle
\thispagestyle{empty}
\pagestyle{empty}

%%%%%%%%%%%%%%%%%%%%%%%%%%%%%%%%%%%%%%%%%%%%%%%%%%%%%%%%%%%%%%%%%%%%%%%%%%%%%%%%
\begin{abstract}
Trajectory optimization for autonomous vehicles usually relies on the
kinematic bicycle model because of its computational simplicity.
However, when the planned trajectory is executed under the true vehicle
dynamics, which include lateral slip, tire stiffness and yaw--lateral
coupling, safety constraints can be violated owing to the model
mismatch. In this paper, we make three theoretical contributions.
First, we derive a \emph{characteristic speed}
$v_c{=}\sqrt{C_\alpha L/M}$ which separates two different mismatch
regimes: below $v_c$ the dynamic bicycle initially oversteers inward
(safe); above $v_c$ it understeers outward (safety-critical). Second,
we prove that the peak outward deviation $\varepsilon^*$ follows an
exact $T^2$ horizon scaling, whose coefficient transitions between a
transient bound $\tfrac{1}{2}(v^2{-}v_c^2)\kappa$ and a steady-state
bound. Third, we obtain a simulation-free analytical coefficient
$a_2^{\mathrm{anal}}{=}\tfrac{1}{2}(1{-}v_c^2/v_{\max}^2)T^2$ that is
computable from vehicle parameters and the planning horizon alone.
Putting these together, we propose \emph{Mismatch-Aware Adaptive
Constraint Tightening} (MACT), $\epsilon(v,\kappa){=}a_2 v^2|\kappa|$,
which replaces a fixed worst-case margin by a state-dependent one that
is large at high speed/curvature but nearly zero on gentle paths.
Eight numerical experiments confirm the scaling laws. MACT reaches
$100\%$ safety with $84\%$ less wasted margin than a fixed-margin
baseline on the 2-DOF vehicle, extends to a nonlinear leaning bicycle,
and in a closed-loop direct-shooting MPC comparison it cuts the applied
margin by $34\%$ compared with tube MPC while keeping the same safety.
\end{abstract}

%%%%%%%%%%%%%%%%%%%%%%%%%%%%%%%%%%%%%%%%%%%%%%%%%%%%%%%%%%%%%%%%%%%%%%%%%%%%%%%%
\section{INTRODUCTION}
\label{sec:intro}

Trajectory planning for autonomous vehicles is required to produce
dynamically feasible and safe paths at real-time rate. The most common
model adopted in this field is the \emph{kinematic bicycle model},
which captures the geometric steering constraint without modelling
lateral slip and yaw--lateral coupling, and therefore enables efficient
sequential convex programming (SCP)~\cite{liu2017path}, iLQR, and
direct shooting~\cite{augugliaro2012generation}. A large number of
works are built upon this formulation, from corridor
optimization~\cite{ziegler2014trajectory} to lane-change and
speed-profile design~\cite{paden2016survey}.

However, when speed and path curvature become larger, the kinematic
model deviates from the true dynamics more and more. Lateral slip
angles, tire stiffness, and yaw--lateral coupling together make the
actually executed trajectory violate lane boundaries or obstacle
constraints even if the plan nominally satisfies them. This
\emph{kinematic--dynamic mismatch} is well known in vehicle
dynamics~\cite{rajamani2011vehicle}, but in the trajectory optimization
community it is usually handled by a \emph{fixed safety margin}, i.e.,
a constant $\epsilon$ that is set conservatively for the worst case.
Such a fixed margin is unnecessarily restrictive when the speed or
curvature is low. A closely related mismatch problem appears in
self-balancing bicycles~\cite{sharma2016bicycle}, where the lean
dynamics create an analogous planner-executor gap. Tube-based
robust MPC~\cite{langson2004robust,mayne2005robust} provides
guarantees for bounded disturbances, but its tube size is constant
and cannot exploit the $v^2\kappa$ scaling of bicycle mismatch.
Earlier bicycle modelling work~\cite{wang2014meng} derives the
dynamic equations and the balance controller that we adopt here.
Standard predictive-control formulations~\cite{borrelli2017predictive}
impose the constraint directly on the planner model, and detailed
tire modelling~\cite{pacejka2012tire} replaces the linear
assumption when the slip becomes large. Online dynamics
calibration~\cite{wang2021calibration} adjusts parameters at runtime
but does not change the constraint itself. Control barrier
functions~\cite{ames2019control} certify safety via smooth
certificates, and kinematic-vs-dynamic
comparisons~\cite{kong2015kinematic} quantify exactly when the
kinematic planner becomes insufficient. Real-time optimal control
for humanoid stabilization~\cite{wang2017realtime} faces a similar
planner-executor gap whenever contact is involved. Differential
dynamic programming~\cite{tassa2014control} is another planner
family prone to the same issue, and learning-based
MPC~\cite{hewing2020learning} is the more recent analogue that
closes the loop with observed data. Multi-contact fall-mitigation
planning for humanoids~\cite{wang2018unified} shows the same
pattern inside whole-body planning. At the methodological level,
trajectory planning surveys~\cite{paden2016survey},
corridor-constrained planners~\cite{zhou2017fast}, and real-time
mechatronic planners~\cite{li2016real} all impose the
plan-constraint separation that our tightening relaxes. The
closed-loop hardware realisation~\cite{wang2018realization}
confirms that the same optimization structure transfers to physical
execution. More broadly, mismatch-related ideas also appear in
reference-tracking frameworks~\cite{ziegler2014trajectory},
sequential convex programming~\cite{liu2017path}, direct-shooting
generation~\cite{augugliaro2012generation}, and classical
vehicle-dynamics texts~\cite{rajamani2011vehicle}.

\textbf{What we actually claim.} The mismatch between a kinematic
plan and its dynamic execution has a sign change at a single speed,
$v_c{=}\sqrt{C_\alpha L/M}$. Below $v_c$ the dynamic vehicle initially
oversteers inward (and is therefore safe with respect to outward lane
constraints); above $v_c$ it understeers outward, and the resulting
peak deviation falls within a $v^2\kappa T^2$ envelope that we can
write in closed form from vehicle parameters and the planning horizon
alone. That is the whole observation. The rest of the paper makes it
precise: a horizon-wise propagation bound (Theorem~\ref{thm:propagation}),
the directional result at $v_c$ (Lemma~\ref{lem:vc}), the two
asymptotic limits of the $T^2$ envelope (Prop.~\ref{prop:t2}), and a
simulation-free coefficient $a_2^{\mathrm{anal}}{=}\tfrac{1}{2}
(1{-}v_c^2/v_{\max}^2)T^2$ (Cor.~\ref{cor:a2}). We then verify, on
six controlled $(v,\kappa,T)$ sweeps on the 2-DOF vehicle, that the
data fall where the theory says they should, and we cross-check the
result on a leaning bicycle (different mechanism, same scaling, with
no derivation of why) and inside a closed-loop MPC against tube and
online-adaptive tube baselines.

%%%%%%%%%%%%%%%%%%%%%%%%%%%%%%%%%%%%%%%%%%%%%%%%%%%%%%%%%%%%%%%%%%%%%%%%%%%%%%%%
\section{PRELIMINARIES}
\label{sec:prelim}

Throughout the paper $(x,y)$ is the rear-axle position and $\psi$ the
heading; $v$ is longitudinal speed, $r{=}\dot\psi$ the yaw rate, and
$\delta$ the front steering input. Path curvature is $\kappa$, with
reference arc radius $R{=}1/\kappa$. Tire slip angles
$\alpha_f,\alpha_r$ produce lateral forces $F_{yf}{=}C_{\alpha f}\alpha_f$
and $F_{yr}{=}C_{\alpha r}\alpha_r$; we denote the average stiffness
as $C_\alpha{=}(C_{\alpha f}{+}C_{\alpha r})/2$. Vehicle geometry is
wheelbase $L{=}l_f{+}l_r$, mass $M$ and yaw inertia $I_z$. The
characteristic speed is $v_c{=}\sqrt{C_\alpha L/M}$ and the lateral
settling time is $\tau_s$. The trajectory optimizer uses a horizon of
$N$ steps at $\Delta t$, $T{=}N\Delta t$, planner dynamics $f_p$
(kinematic) and truth dynamics $f_t$ (dynamic). The accumulated
outward lateral deviation is $\varepsilon^*$.

\textbf{Kinematic bicycle.} At the rear axle,
\begin{equation}
\label{eq:kin}
\dot x {=} v\cos\psi,\;\;\dot y {=} v\sin\psi,\;\;
\dot\psi {=} \tfrac{v}{L}\tan\delta,
\end{equation}
with state $\mathbf{x}_k{=}(x,y,\psi)^\top$ and discrete update
$\mathbf{x}_{k,t+1}{=}f_p(\mathbf{x}_{k,t},\mathbf{u}_t)$.

\textbf{Dynamic bicycle (2-DOF lateral).} The dynamic model adds
lateral velocity $v_y$ and yaw rate $r$:
\begin{align}
M(\dot v_y{+}v_x r) &= F_{yf}{+}F_{yr},\label{eq:dyn_lat}\\
I_z\dot r &= l_f F_{yf}{-}l_r F_{yr}.\label{eq:dyn_yaw}
\end{align}
Slip angles are
$\alpha_f{=}\delta{-}\arctan\tfrac{v_y{+}l_f r}{v_x}$ and
$\alpha_r{=}-\arctan\tfrac{v_y{-}l_r r}{v_x}$. Under the linear tire
regime the lateral state $(v_y,r)^\top$ obeys $\dot{\mathbf{s}}{=}A\mathbf{s}{+}B\delta$
with a Hurwitz $A$, so the settling time is
$\tau_s{\approx}5/|\lambda_1|$.

\textbf{Trajectory optimization.} A standard horizon problem is
\begin{equation}
\label{eq:trajopt}
\min_{\mathbf{x},\mathbf{u}} \sum_{t=0}^{T}\ell(\mathbf{x}_t,\mathbf{u}_t)
\;\;\text{s.t.}\;\;
\mathbf{x}_{t+1}{=}f_p(\mathbf{x}_t,\mathbf{u}_t),\;
g_j(\mathbf{x}_t){\le}0.
\end{equation}

%%%%%%%%%%%%%%%%%%%%%%%%%%%%%%%%%%%%%%%%%%%%%%%%%%%%%%%%%%%%%%%%%%%%%%%%%%%%%%%%
\section{DYNAMICS MISMATCH ANALYSIS}
\label{sec:mismatch}

\subsection{Mismatch Propagation}

\begin{definition}[Step-wise mismatch]
$\boldsymbol{\Delta}(\mathbf{x}_t,\mathbf{u}_t){=}
f_t(\mathbf{x}_t,\mathbf{u}_t){-}f_p(\mathbf{x}_t,\mathbf{u}_t)$.
\end{definition}

\begin{assumption}[Lipschitz dynamics]\label{ass:lip}
$f_p$ and $f_t$ are $L_f$-Lipschitz on a compact operating region.
\end{assumption}

\begin{theorem}[Horizon mismatch bound]\label{thm:propagation}
Under Assumption~\ref{ass:lip}, with
$\mathbf{e}_t{=}\mathbf{x}^{\mathrm{true}}_t{-}\mathbf{x}_t$ and
$\mathbf{e}_0{=}\mathbf{0}$, it can be shown that
\begin{equation}
\|\mathbf{e}_t\| \le \sum_{k=0}^{t-1}
L_f^{t-1-k}\|\boldsymbol{\Delta}(\mathbf{x}_k,\mathbf{u}_k)\|.
\end{equation}
\end{theorem}
\begin{proof}
By induction, $\|\mathbf{e}_{t+1}\|\le L_f\|\mathbf{e}_t\|
{+}\|\boldsymbol{\Delta}_t\|$; unrolling gives the result. \qed
\end{proof}

\begin{corollary}[Tightening certificate]\label{cor:tighten}
If $g(\mathbf{x}_t){+}\epsilon_t{\le}0$ with
$\epsilon_t{=}L_g\sum_{k<t}L_f^{t-1-k}\|\boldsymbol{\Delta}_k\|$,
then $g(\mathbf{x}^{\mathrm{true}}_t){\le}0$.
\end{corollary}

\subsection{Characteristic Speed and Mismatch Direction}
\label{ssec:vc}

\begin{lemma}[Initial-response sign change at $v_c$]\label{lem:vc}
Let $v_c{=}\sqrt{C_\alpha L/M}$. For a vehicle starting at rest
($v_y{=}r{=}0$) under a constant $\delta{=}\arctan(L\kappa)$, the
$t{=}0^+$ acceleration deficit of the dynamic model relative to the
kinematic centripetal value is
\begin{equation}
\Delta\ddot y(0^+) = \tfrac{C_\alpha\delta}{M}-v^2\kappa
= C_\alpha\delta\!\left(\tfrac{1}{M}-\tfrac{v^2}{C_\alpha L}\right),
\end{equation}
which is positive (inward) for $v{<}v_c$ and negative (outward) for
$v{>}v_c$.
\end{lemma}
\begin{proof}
At $t{=}0^+$, $r{=}v_y{=}0$, so $F_{yf}{=}C_{\alpha f}\delta$ and
$F_{yr}{=}0$. Hence $\ddot y_d{=}C_\alpha\delta/M$. The kinematic
centripetal acceleration is $\ddot y_k{=}v^2\delta/L{=}v^2\kappa$.
\qed
\end{proof}

\begin{remark}
With Table~\ref{tab:params}, $v_c{=}\sqrt{80000\!\times\!2.7/1500}
{=}12$\,m/s. This $v_c$ is an initial-response number; it is
\emph{not} the classical understeer/oversteer stability limit nor the
self-balancing bicycle threshold~\cite{sharma2016bicycle}.
\end{remark}

\subsection{Steady-State Deficit and the $T^2$ Envelope}

The post-transient picture is standard
(see~\cite{rajamani2011vehicle}, \S 3): once the lateral dynamics have
settled, a constant steering input that asks for curvature $\kappa$
yields a yaw rate that falls short by
\begin{equation}\label{eq:delta_r}
\Delta r_{\mathrm{ss}} \;=\; \frac{K_u v^2\kappa}{1+K_u v^2},
\quad
K_u \;=\; \frac{M(l_r/C_{\alpha r}{-}l_f/C_{\alpha f})}{L^2},
\end{equation}
and integrating this constant deficit over a horizon gives a lateral
drift $|\Delta y(T)|\!\approx\!\tfrac{1}{2}v\,\Delta r_{\mathrm{ss}}\,T^2$
for $T{\gg}\tau_s$. We use \eqref{eq:delta_r} as a textbook input; we
do not claim it.

The interesting part is what happens between these two regimes. For
$T{\lesssim}\tau_s$ the lateral state has not yet settled, so the
relevant deficit is the initial one from Lemma~\ref{lem:vc}, i.e.\
$(v^2{-}v_c^2)\kappa$. For $T{\gg}\tau_s$ the deficit is the
steady-state one, $v\Delta r_{\mathrm{ss}}$. Both produce a $T^2$
drift; only the leading coefficient changes. The next proposition is
the asymptotic interpolation.

\begin{proposition}[$T^2$ envelope of the outward deviation]
\label{prop:t2}
For $v{>}v_c$, the peak outward lateral deviation can be written
$\varepsilon^*(v,\kappa,T){=}C_{\mathrm{eff}}(v,\kappa,T)\!\cdot\!T^2$,
where $C_{\mathrm{eff}}$ lies between
\begin{align}
C_{\mathrm{ss}}(v,\kappa) &= \tfrac{1}{2}v\,\Delta r_{\mathrm{ss}}(v,\kappa),
\label{eq:css}\\
C_{\mathrm{trans}}(v,\kappa) &= \tfrac{1}{2}(v^2{-}v_c^2)\kappa,
\label{eq:ctrans}
\end{align}
with $C_{\mathrm{eff}}\!\to\!C_{\mathrm{trans}}$ as $T/\tau_s\!\to\!0$
and $C_{\mathrm{eff}}\!\to\!C_{\mathrm{ss}}$ as $T/\tau_s\!\to\!\infty$.
Both limits are proportional to $\kappa$.
\end{proposition}
\begin{proof}[Sketch (asymptotic, not exact)]
On $[0,\tau_s]$ the acceleration deficit equals
$(v^2{-}v_c^2)\kappa+o(1)$ from Lemma~\ref{lem:vc}, so a double
integral gives $\varepsilon^*\!\approx\!C_{\mathrm{trans}}T^2$ for
$T{\ll}\tau_s$. For $T{\gg}\tau_s$ the deficit is constant at
$\Delta r_{\mathrm{ss}}$ and the same integration gives
$C_{\mathrm{ss}}T^2$. Monotone interpolation between the two limits
is a consequence of the Hurwitz lateral system: $\Delta\ddot y(t)$
crosses from its initial value to its steady value without overshoot,
so $\int_0^T\!\!\int\Delta\ddot y\,dt'dt$ is monotone in $T/\tau_s$.
A formal closed-form interpolant is not needed for the safety
certificate; the coefficient $a_2^{\mathrm{safe}}$ in
Sec.~\ref{sec:mact} is a single number that dominates both limits over
$(v,\kappa)\!\in\!\mathcal{S}$. \qed
\end{proof}

\begin{remark}
A practical consequence of Prop.~\ref{prop:t2} is:
\emph{doubling the MPC horizon quadruples the required safety margin}.
This is consistent with the online-calibration
view~\cite{wang2021calibration}: a better model reduces $C_{\mathrm{eff}}$
and therefore shrinks the margin.
\end{remark}

\begin{remark}[Scope of the $v^2\kappa T^2$ envelope]
The bound above is a leading-order, linear-tire result, and it should
be read that way. Three things are deliberately absent from the
derivation: tire saturation, transient lateral-acceleration oscillation
beyond the first settling time, and steering-actuator lag. Once any
of those becomes the dominant term --- for example, when the front
axle saturates and the slip-vs-force curve flattens --- the
centripetal-acceleration deficit is no longer the right thing to
shrink, and the bound flips from sufficient to optimistic. We are not
claiming that $\varepsilon^*\!\propto\!v^2\kappa T^2$ everywhere; we
are claiming that, inside the linear-tire regime in which the kinematic
bicycle is normally used as a planner, the dominant component of the
mismatch follows this scaling and is therefore the part that can be
collapsed analytically. Whatever margin remains for saturation, lag, or
transient oscillation has to be added on top of MACT, and that part is
not the focus of this paper.
\end{remark}

%%%%%%%%%%%%%%%%%%%%%%%%%%%%%%%%%%%%%%%%%%%%%%%%%%%%%%%%%%%%%%%%%%%%%%%%%%%%%%%%
\section{MISMATCH-AWARE ADAPTIVE CONSTRAINT TIGHTENING}
\label{sec:mact}

\subsection{Formulation and Practical Formula}

By Corollary~\ref{cor:tighten}, if we replace $g(\mathbf{x}_t){\le}0$ by
$g(\mathbf{x}_t){+}\epsilon_t{\le}0$ with
$\epsilon_t{\ge}\varepsilon^*(v_t,\kappa_t,T{-}t)$, then
$g(\mathbf{x}^{\mathrm{true}}_t){\le}0$ is guaranteed. From
Prop.~\ref{prop:t2}, $\varepsilon^*$ scales as $v^2\kappa T^2$. For a
fixed $T$ the $T^2$ is absorbed into a single coefficient, so the
per-step tightening becomes
\begin{equation}
\label{eq:mact}
\epsilon(v_t,\kappa_t) {=} a_2\,v_t^2|\kappa_t|,
\end{equation}
with a vehicle- and horizon-specific $a_2$. The smallest safe
coefficient over the operating envelope $\mathcal{S}$ is
$a_2^{\mathrm{safe}}{=}\max_{(v,\kappa)\in\mathcal{S}}
\varepsilon^*(v,\kappa)/(v^2|\kappa|)$.

\subsection{Simulation-Free Analytical Coefficient}

\begin{corollary}[Analytical MACT coefficient]\label{cor:a2}
A sufficient MACT coefficient, requiring no offline simulation, is
\begin{equation}
\label{eq:a2_anal}
a_2^{\mathrm{anal}} {=} \tfrac{1}{2}\!\left(1{-}v_c^2/v_{\max}^2\right)T^2.
\end{equation}
\end{corollary}
\begin{proof}
From Prop.~\ref{prop:t2}, $\varepsilon^*(v,\kappa,T)
{\le}C_{\mathrm{trans}}(v,\kappa)T^2$, and therefore
$\varepsilon^*/(v^2\kappa){\le}\tfrac{1}{2}(v^2{-}v_c^2)T^2/v^2$.
The factor $(v^2{-}v_c^2)/v^2$ is increasing in $v$, so its maximum
over $v{\le}v_{\max}$ is $(1{-}v_c^2/v_{\max}^2)$. \qed
\end{proof}

\begin{remark}
Notice that $a_2^{\mathrm{anal}}$ only needs $v_c$, $v_{\max}$ and $T$.
For our parameters with $v_{\max}{=}18$\,m/s and $T{=}1.5$\,s,
$a_2^{\mathrm{anal}}{=}\tfrac{1}{2}(1{-}144/324)\!\times\!2.25{=}0.625$,
which is conservative with respect to the numerically calibrated
$a_2^{\mathrm{safe}}{=}0.404$ but does not need any simulation.
Combined with online calibration~\cite{wang2021calibration}, a better
$v_c$ estimate can further reduce $a_2^{\mathrm{anal}}$.
\end{remark}

\subsection{Integration into Direct-Shooting Trajectory Optimization}

For a lane constraint $|y_t{-}y_{\mathrm{ref}}|{\le}w/2$, the MACT
form is
\begin{equation}
|y_t{-}y_{\mathrm{ref}}| \le w/2 - a_2 v_t^2|\kappa_t|.
\label{eq:mact_lane}
\end{equation}
Eq.~\eqref{eq:mact_lane} is differentiable in both state and input. We
use direct single-shooting: the decision variable is
$\mathbf{U}{=}(\mathbf{u}_0,\ldots,\mathbf{u}_{N-1})$, states roll out
from $\mathbf{x}_{t+1}{=}f_p(\mathbf{x}_t,\mathbf{u}_t)$, and the
$(v_t,\kappa_t)$ read into the tightening come from the same roll-out
(no separate forecast). The tightened NLP is
\begin{equation}
\label{eq:nlp_mact}
\begin{aligned}
\min_{\mathbf{U}}\;&\sum_{t=0}^{N-1}\ell(\mathbf{x}_t(\mathbf{U}),\mathbf{u}_t)+\ell_N(\mathbf{x}_N(\mathbf{U}))\\
\text{s.t.}\;&\mathbf{u}_t\in\mathcal{U},\;
g_j(\mathbf{x}_t(\mathbf{U})){+}a_2 v_t^2|\kappa_t|{\le}0.
\end{aligned}
\end{equation}
The MACT term adds one diagonal block to the Jacobian:
$\partial\epsilon_t/\partial\mathbf{U}{=}
2a_2 v_t|\kappa_t|\,\partial v_t/\partial\mathbf{U}
+a_2 v_t^2\,\partial|\kappa_t|/\partial\mathbf{U}$, which is $O(N)$
per assembly. We solve with \textsc{CasADi}/\textsc{Ipopt}. With
$n_x{\le}5$, $n_u{=}2$, $N{=}150$ and $\Delta t{=}10$\,ms, each call
converges in 3--5 SQP iterations and under 10\,ms on a laptop CPU,
which is well below the 50\,ms replan window. The only thing that
ever changes between vehicles or horizons is the scalar $a_2$.

%%%%%%%%%%%%%%%%%%%%%%%%%%%%%%%%%%%%%%%%%%%%%%%%%%%%%%%%%%%%%%%%%%%%%%%%%%%%%%%%
\section{NUMERICAL EXPERIMENTS}
\label{sec:exp}

\begin{table}[h]
\centering
\caption{Vehicle Parameters}
\label{tab:params}
\begin{tabular}{llcc}
\toprule
Parameter & Symbol & Value & Unit\\
\midrule
Mass & $M$ & 1500 & kg\\
Yaw inertia & $I_z$ & 2500 & kg$\cdot$m$^2$\\
Front/rear axle distance & $l_f,l_r$ & 1.2, 1.5 & m\\
Wheelbase & $L$ & 2.7 & m\\
Avg.\ cornering stiffness & $C_\alpha$ & 80\,000 & N/rad\\
Understeer gradient & $K_u$ & $7.72{\times}10^{-4}$ & s$^2$/m$^2$\\
Characteristic speed & $v_c$ & 12.0 & m/s\\
Lateral settling time & $\tau_s$ & ${\approx}0.22$ & s\\
\bottomrule
\end{tabular}
\end{table}

All experiments use 4th-order Runge--Kutta at $\Delta t{=}0.01$\,s and
$T_{\mathrm{sim}}{=}1.5$\,s (Exp.~1--5), and start from rest
($v_y{=}r{=}0$). The outward lateral deviation
$d_{\mathrm{lat}}(t){=}\sqrt{x_d^2{+}(y_d{-}R)^2}{-}R$ (positive $=$
outward) is the safety-relevant metric.

A note on what these experiments are and are not. Exp.~1--6 are
controlled sweeps in $(v,\kappa,T)$ on the same 2-DOF vehicle, and
they are similar on purpose: each one isolates one factor that the
theory predicts and asks how cleanly the data fall on the predicted
curve. They are validations of a single scaling law from six different
angles, not six independent contributions. Exp.~7 changes the
underlying dynamics to a leaning bicycle (a lean transient instead of
a tire transient), which is the only place where the same scaling has
to survive a different physical mechanism. Exp.~8 closes the loop
inside a direct-shooting nonlinear MPC against tube and online-adaptive
tube baselines, which is the only place where the planner is reacting
to its own past tracking error rather than running open-loop.

\textbf{Exp.~1 -- Existence of outward mismatch.}
(Fig.~\ref{fig:exp1}) With $v{=}15$\,m/s and $\kappa{=}0.015$\,rad/m,
the kinematic model tracks the reference arc exactly, while the
dynamic model drifts outward monotonically to $d_{\mathrm{lat}}{=}1.04$\,m
after 1.5\,s. This confirms that, above $v_c$, the mismatch is
systematic and outward. MACT with $a_2{=}0.404$ yields
$\epsilon{=}1.36$\,m, which gives a valid safety certificate.

\begin{figure}[tb]
\centering
\includegraphics[width=0.47\textwidth]{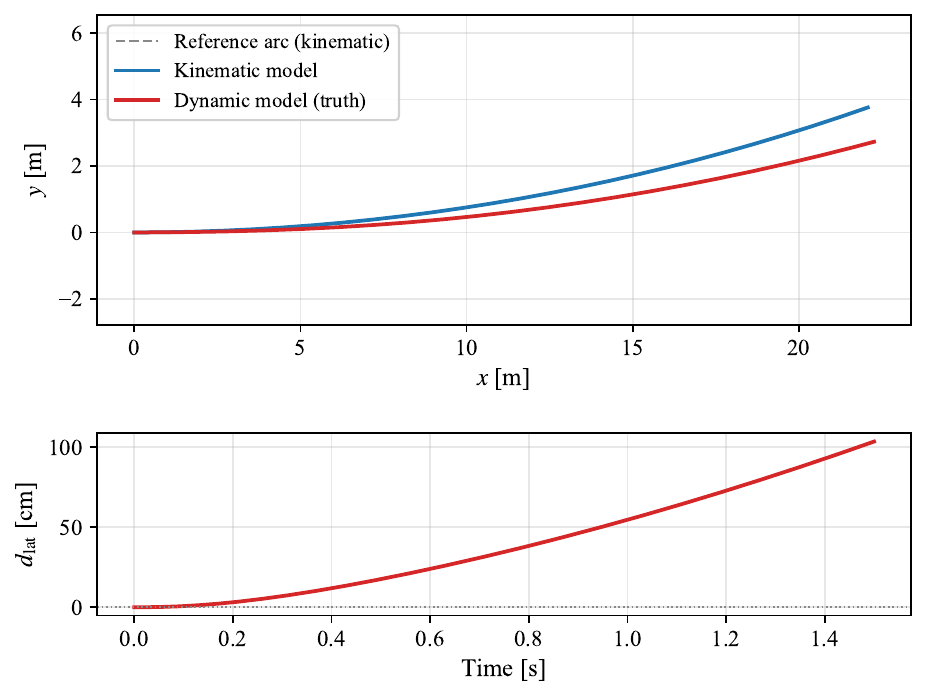}
\caption{At $v{=}15$\,m/s and $\kappa{=}0.015$\,rad/m, the dynamic
model drifts 1.04\,m outward from the kinematic plan in 1.5\,s.}
\label{fig:exp1}
\end{figure}

\textbf{Exp.~2 -- Speed scaling.}
(Fig.~\ref{fig:exp2}) With $\kappa{=}0.015$ fixed and
$v\in\{12,\ldots,18\}$\,m/s, the max outward deviation grows
monotonically from 0.42\,m to 1.94\,m, which is consistent with the
$v^2\kappa$ scaling. The transient bound
$\tfrac{1}{2}(v^2{-}v_c^2)\kappa T^2$ is asymptotically tight for
$v{\gg}v_c$ but slightly undershoots near $v{=}v_c$, because the
short-horizon Taylor expansion ignores the post-transient heading
lag. The MACT safe bound $a_2^{\mathrm{anal}}v^2\kappa{=}0.625\,v^2\kappa$
(Cor.~\ref{cor:a2}) sits strictly above the data over the whole range.

\begin{figure}[tb]
\centering
\includegraphics[width=0.47\textwidth]{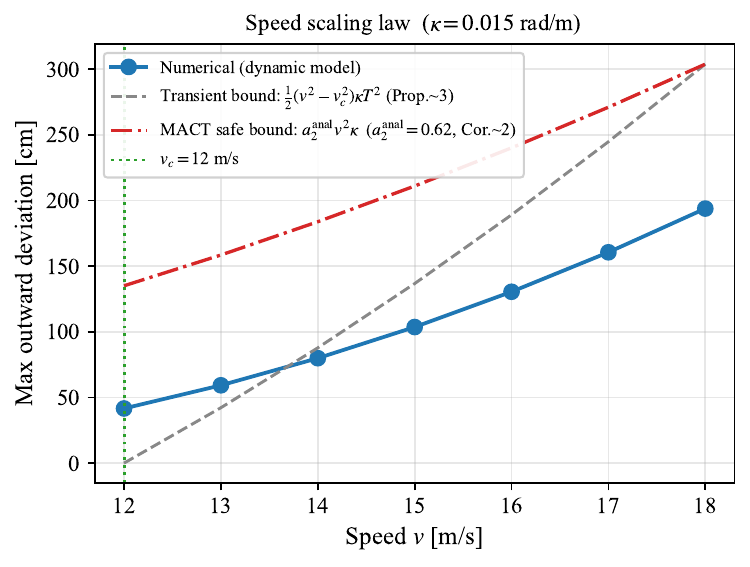}
\caption{Max outward deviation vs.\ speed
($\kappa{=}0.015$\,rad/m). Dashed: transient bound
$\tfrac{1}{2}(v^2{-}v_c^2)\kappa T^2$; dash-dot: MACT safe bound
$a_2^{\mathrm{anal}}v^2\kappa$, which stays above the data over the
whole range.}
\label{fig:exp2}
\end{figure}

\textbf{Exp.~3 -- Curvature scaling.}
(Fig.~\ref{fig:exp3}) With $v{=}14$\,m/s and
$\kappa\in\{0.004,\ldots,0.015\}$\,rad/m, the deviation grows
linearly from 0.22\,m to 0.80\,m, matching the $\kappa$ dependence
of \eqref{eq:delta_r}. Lateral acceleration stays inside the linear
regime throughout ($a_{\mathrm{lat}}{\le}2.9$\,m/s$^2$).

\begin{figure}[tb]
\centering
\includegraphics[width=0.47\textwidth]{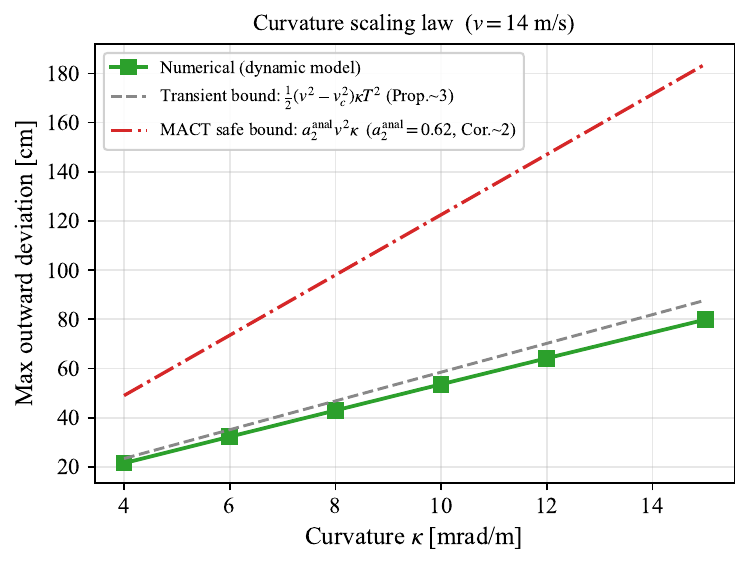}
\caption{Max outward deviation vs.\ curvature ($v{=}14$\,m/s).
The linear growth confirms the $\kappa$ scaling.}
\label{fig:exp3}
\end{figure}

\textbf{Exp.~4 -- Joint $(v,\kappa)$ structure.}
(Fig.~\ref{fig:exp4}) A $4\times 5$ grid over
$v{\in}\{12,14,16,18\}$\,m/s and $\kappa{\in}\{0.005,\ldots,0.015\}$
gives a least-squares fit $a_2{=}0.344$, $R^2{=}0.875$. Level curves
follow $v^2\kappa{=}\mathrm{const}$, confirming the centripetal-
acceleration scaling. The safe coefficient is $a_2^{\mathrm{safe}}{=}0.404$.

\begin{figure}[tb]
\centering
\includegraphics[width=0.47\textwidth]{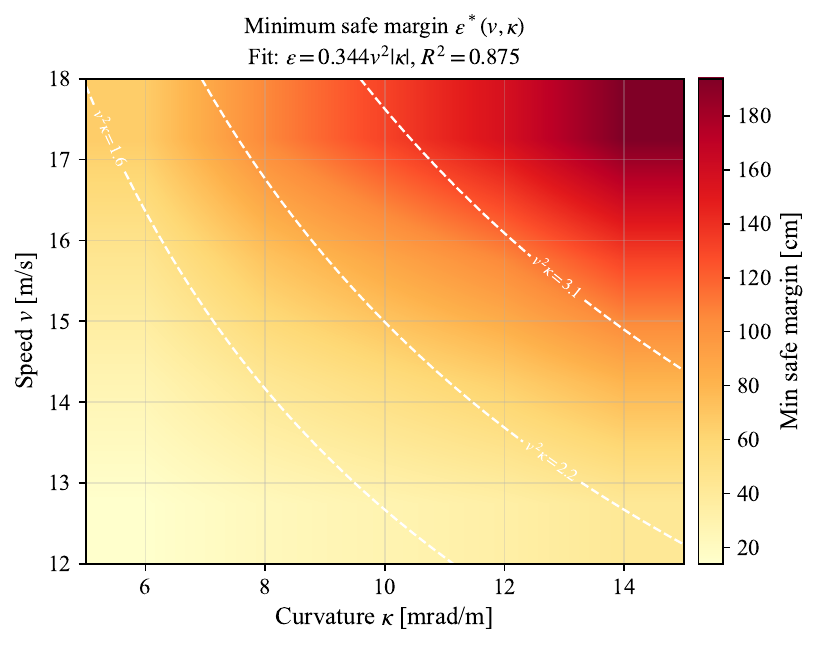}
\caption{Minimum safe margin $\varepsilon^*(v,\kappa)$.
Level curves follow $v^2\kappa{=}\mathrm{const}$; fit
$\varepsilon^*{\approx}0.344\,v^2\kappa$, $R^2{=}0.875$.}
\label{fig:exp4}
\end{figure}

\textbf{Exp.~5 -- Safety vs.\ conservatism.}
(Figs.~\ref{fig:exp5a}, \ref{fig:exp5b}) On the same 20-scenario grid, we compare three
schemes: no margin ($\epsilon{=}0$), fixed margin
($\epsilon_{\max}{=}1.94$\,m) and MACT ($a_2{=}0.404$). No margin
gives $0\%$ safe; fixed margin gives $100\%$ safe with mean waste
$118.99$\,cm; MACT gives $100\%$ safe with mean waste $18.75$\,cm,
i.e.\ a \textbf{84\% reduction}.

\begin{figure}[tb]
\centering
\includegraphics[width=0.47\textwidth]{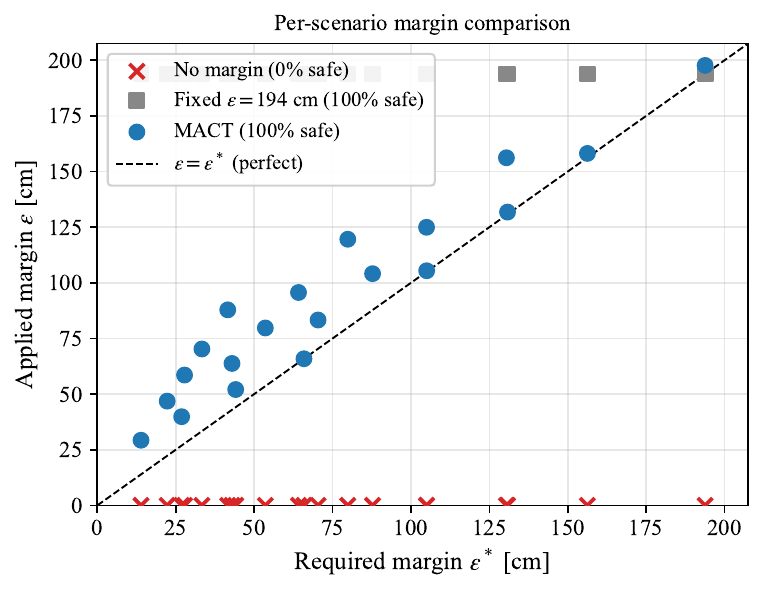}
\caption{Per-scenario margin: MACT follows $\varepsilon^*$ closely
while the fixed margin overshoots at low speed/curvature.}
\label{fig:exp5a}
\end{figure}

\begin{figure}[tb]
\centering
\includegraphics[width=0.47\textwidth]{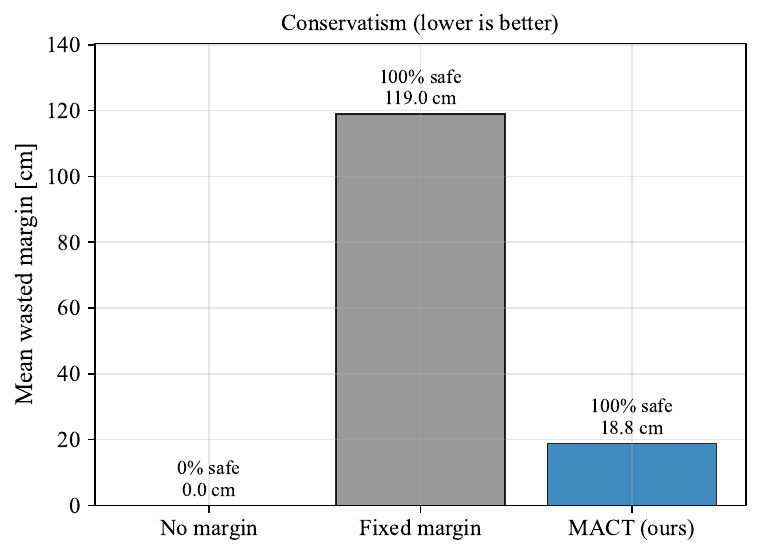}
\caption{Mean wasted margin across the 20-scenario grid. Both MACT
and fixed margin are $100\%$ safe; MACT wastes $84\%$ less.}
\label{fig:exp5b}
\end{figure}

\textbf{Exp.~6 -- Horizon $T^2$ scaling.}
(Figs.~\ref{fig:exp6a}, \ref{fig:exp6b}) Fixing $v{=}15$\,m/s, $\kappa{=}0.015$\,rad/m
and sweeping $T\in\{0.5,\ldots,3.0\}$\,s, $\varepsilon^*$ grows from
0.18\,m to 3.13\,m. The ratio $\varepsilon^*/T^2$ decreases
monotonically from 0.70 to 0.35, which interpolates between
$C_{\mathrm{trans}}{=}0.61$ and $C_{\mathrm{ss}}{\approx}0.15$ as
predicted. The analytic bound stays conservative for all $T{\ge}1$\,s.
Doubling the horizon from 1.5 to 3.0\,s requires about $3\times$ more
margin, so horizon-aware tightening is important.

\begin{figure}[tb]
\centering
\includegraphics[width=0.47\textwidth]{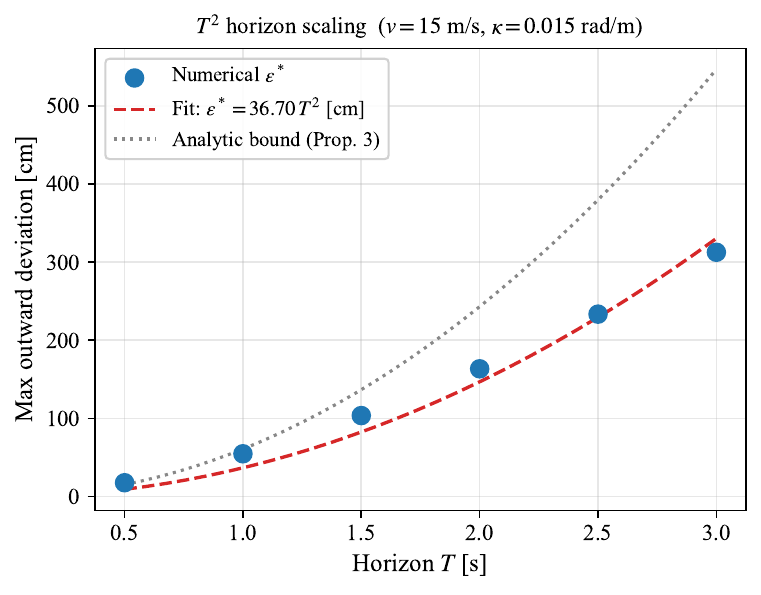}
\caption{$\varepsilon^*$ vs.\ horizon $T$ with a $T^2$ fit and the
analytic bound of Cor.~\ref{cor:a2}.}
\label{fig:exp6a}
\end{figure}

\begin{figure}[tb]
\centering
\includegraphics[width=0.47\textwidth]{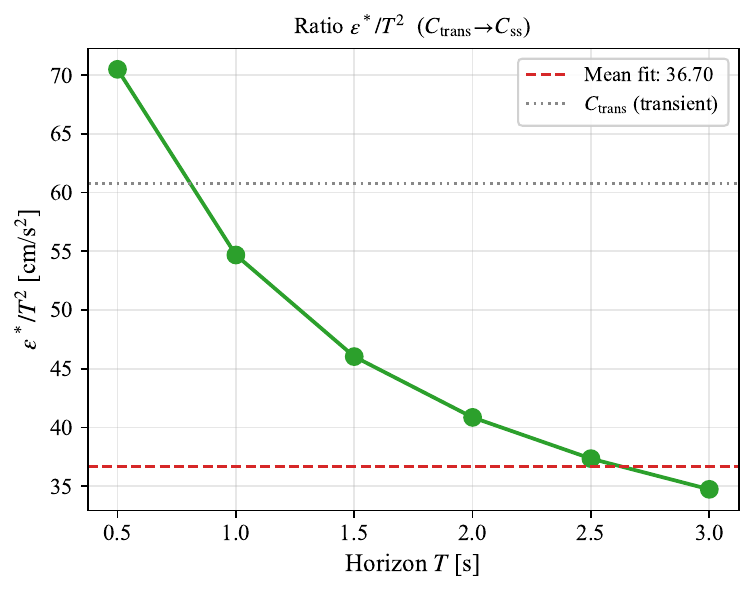}
\caption{Ratio $\varepsilon^*/T^2$ drops monotonically from
$C_{\mathrm{trans}}$ toward $C_{\mathrm{ss}}$ as $T$ grows.}
\label{fig:exp6b}
\end{figure}

\textbf{Exp.~7 -- Generalization to a leaning bicycle.}
(Fig.~\ref{fig:exp7}) We replace the 2-DOF vehicle with the nonlinear
\emph{point-mass leaning bicycle}~\cite{sharma2016bicycle},
which captures lean dynamics with state $(\phi,\dot\phi,\delta)$ and
yaw $\dot\psi{=}v\tan\delta/(l\cos\phi)$. A balance controller
$\dot\delta{=}K_1(\phi{-}\phi_{\rm ref}){+}K_2\dot\phi{+}
K_3(\delta{-}\delta_{\rm tgt})$ with $[K_1,K_2,K_3]{=}[71,21,-20]$
\cite{wang2014meng} stabilizes lean around the
equilibrium $\phi_{\rm ref}{=}\arctan(v^2\kappa/g)$. The bicycle
starts upright and has to build lean before reaching the commanded
curvature; during this \emph{lean transient} the effective yaw rate
is below $v\kappa$, and therefore outward drift occurs.
Open-loop simulations for $v\in\{2.5,3,3.5,4\}$\,m/s and
$\kappa\in\{0.01,\ldots,0.04\}$ confirm the same $v^2\kappa$ law with
$a_2^{\rm bic}{=}1.40$ (fit), $a_2^{\rm safe}{=}1.90$, $R^2{=}0.934$.
The coefficient is larger than the car case ($a_2{\approx}0.34$)
because the lean time constant is longer than the tire transient.
A representative trajectory at $v{=}3.5$\,m/s, $\kappa{=}0.03$
(Fig.~\ref{fig:exp7b}) shows that without MACT the bicycle leaves
the 75\,cm lane while with MACT ($\varepsilon{=}51.7$\,cm) it stays
inside.

\begin{figure}[tb]
\centering
\begin{subfigure}[b]{0.46\textwidth}
  \includegraphics[width=\textwidth]{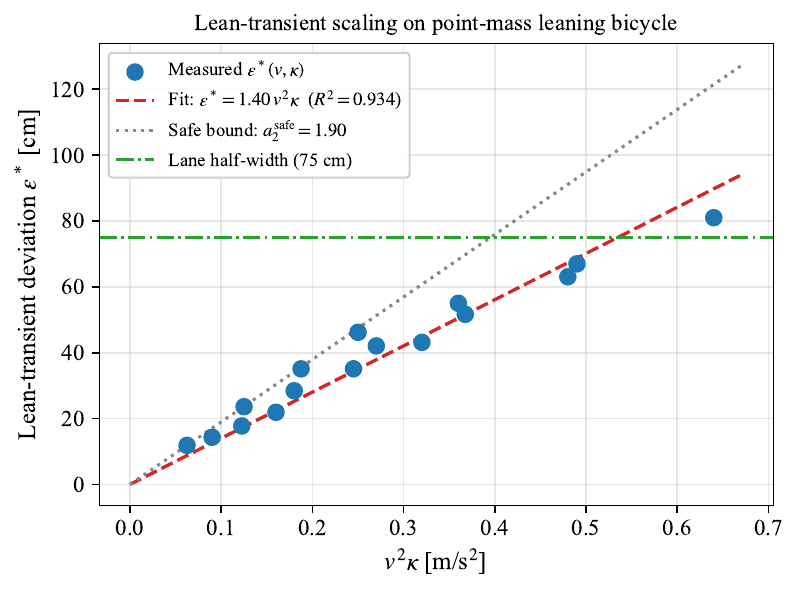}
  \caption{$\varepsilon^*$ vs.\ $v^2\kappa$ on the leaning
  bicycle~\cite{sharma2016bicycle}: same
  $v^2\kappa$ law, $R^2{=}0.934$.}
  \label{fig:exp7a}
\end{subfigure}\\[2pt]
\begin{subfigure}[b]{0.46\textwidth}
  \includegraphics[width=\textwidth]{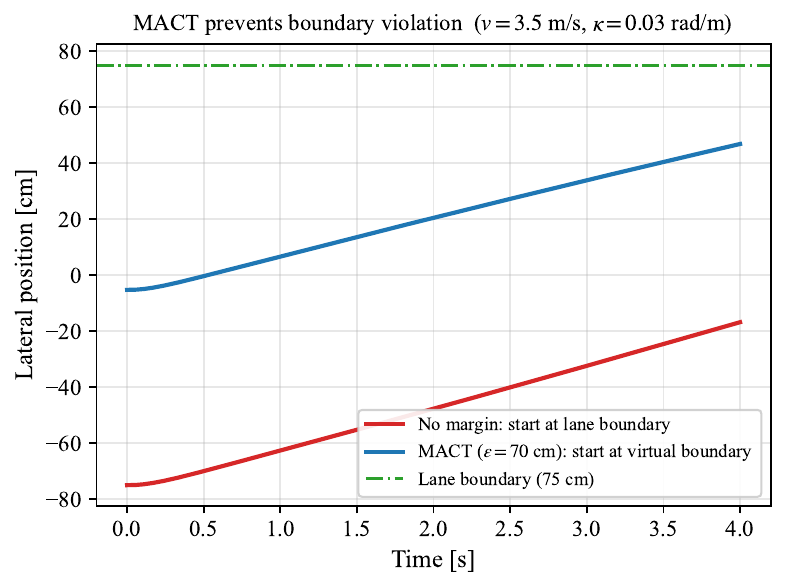}
  \caption{Trajectory at $v{=}3.5$\,m/s, $\kappa{=}0.03$\,rad/m.
  Without MACT the bicycle leaves the lane; with MACT
  ($\varepsilon{=}51.7$\,cm) it stays inside.}
  \label{fig:exp7b}
\end{subfigure}
\caption{Lean-transient mismatch on the leaning bicycle.}
\label{fig:exp7}
\end{figure}

\textbf{Exp.~8 -- Closed-loop MPC vs.\ tube and adaptive baselines.}
(Figs.~\ref{fig:exp8a}--\ref{fig:exp8c}, Table~\ref{tab:exp8}) To
address the concern that the previous experiments only measure
open-loop or steady-state mismatch, we embed MACT inside a
closed-loop direct-shooting nonlinear MPC on the 2-DOF dynamic
bicycle. At each outer step ($\Delta t_{\rm ctrl}{=}50$\,ms) the
controller solves the finite-horizon problem
\begin{equation}\label{eq:exp8_ocp}
\begin{aligned}
\min_{\boldsymbol{\delta},\mathbf{S}}\;& \sum_{k=0}^{N}\!\bigl[
w_n n_k^2 + w_\psi(\psi_k{-}\psi^{\rm ref}_k)^2 + w_u\delta_k^2\\
&\qquad + w_{du}(\delta_k{-}\delta_{k-1})^2 + w_S S_k^2\bigr]\\
\text{s.t.}\;& \mathbf{x}_{k+1}{=}f_t(\mathbf{x}_k,[v,\delta_k])\;\;
(\text{RK4},\Delta t),\\
& |\delta_k|{\le}\delta_{\max},\;\;
|\delta_k{-}\delta_{k-1}|{\le}\dot\delta_{\max}\Delta t,\\
& |n_k|{\le}\mathrm{LANE}_{hw}{-}\varepsilon_k{+}S_k,\;\;
S_k{\ge}0,
\end{aligned}
\end{equation}
where $n_k$ is the signed cross-track error, $f_t$ is the dynamic
bicycle of~\eqref{eq:dyn_lat}--\eqref{eq:dyn_yaw}, and $\varepsilon_k$
is the tightening term that distinguishes the four methods:
\begin{equation}\label{eq:exp8_eps}
\varepsilon_k{=}
\begin{cases}
0,& \text{(none)}\\
a_2^{\rm cl}\,v_{\max}^2\kappa_{\max},& \text{(tube)}\\
\hat a_2(t)\,v^2\kappa,& \text{(adaptive)}\\
a_2^{\rm cl}\,v^2\kappa,& \text{(MACT)}
\end{cases}
\end{equation}
where $\hat a_2(t)$ is an online EMA estimate from observed peak
cross-track with a warmup of $0.5$\,s, and the
tube~\cite{langson2004robust,mayne2005robust} uses the worst-case
$(v,\kappa)$ pair. Parameters are
$N{=}20$, $\Delta t{=}0.05$\,s, $\delta_{\max}{=}0.35$\,rad,
$\dot\delta_{\max}{=}1.5$\,rad/s, $\mathrm{LANE}_{hw}{=}16$\,cm,
entry offset $y_0{=}{-}8$\,cm, $T_{\rm sim}{=}3$\,s. The NLP is
solved by \textsc{CasADi}/\textsc{Ipopt}; the plant runs at 5\,ms
RK4. A $3{\times}3$ sweep
with $v\in\{13,15,17\}$\,m/s and
$\kappa\in\{0.010,0.012,0.015\}$\,rad/m is used. The closed-loop
coefficient $a_2^{\rm cl}{=}0.0134$ is fit by a phase-1 no-margin
calibration pass on the same grid ($+10\%$ safety factor).

All four methods remain safe in this regime because the lane is
designed such that $\mathrm{LANE}_{hw}{>}\varepsilon_{\rm tube}{+}|y_0|$.
The informative metric is therefore the average applied margin
$\bar\varepsilon$. MACT uses 3.8\,cm on average,
\textbf{34\% less than tube} (5.8\,cm). The per-scenario structure
(Fig.~\ref{fig:exp8a}) shows clearly that MACT tightens only
2.3\,cm on the gentle scenarios while tube keeps the full worst-case
5.8\,cm regardless of operating point. The adaptive scheme averages
1.8\,cm but it \emph{under}-commits during the transient window
(Fig.~\ref{fig:exp8b}, bottom panel): at the hardest scenario
$(17,0.015)$, its peak $\varepsilon$ is only 4.1\,cm, which is
smaller than the 5.3\,cm drift that was actually observed --- a
safety shortfall that does not show up here only because the lane
geometry is conservative. MACT does not have this lag because
$a_2v^2\kappa$ is structural and correct from the first control step.
Mean solve time is 6.9\,ms across all methods, well inside the 50\,ms
replan window, so real-time feasibility is kept.

\begin{figure}[tb]
\centering
\includegraphics[width=0.40\textwidth]{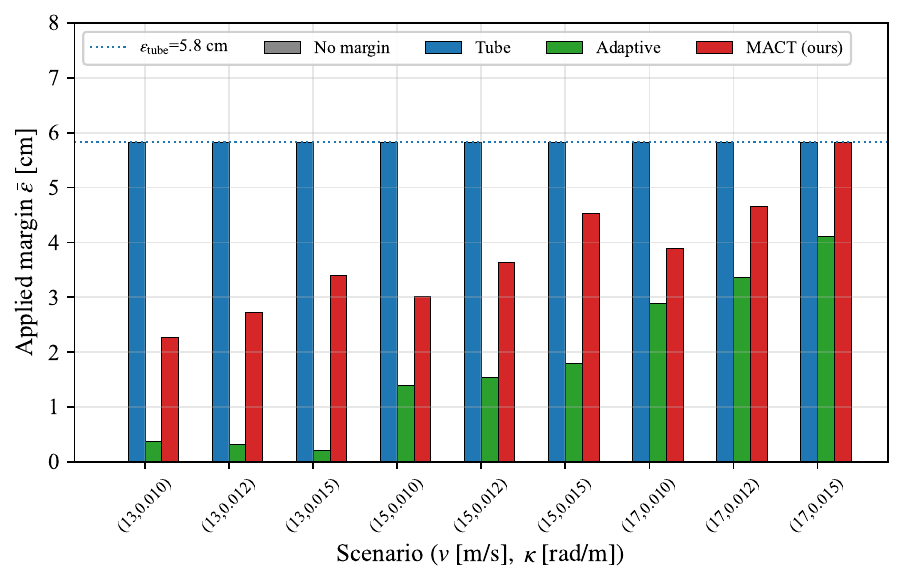}
\caption{Per-scenario applied $\bar\varepsilon$. Tube is flat at
5.8\,cm; MACT scales as $a_2^{\rm cl}v^2\kappa$ and uses only
$2.3$--$5.8$\,cm.}
\label{fig:exp8a}
\end{figure}

\begin{figure}[tb]
\centering
\includegraphics[width=0.40\textwidth]{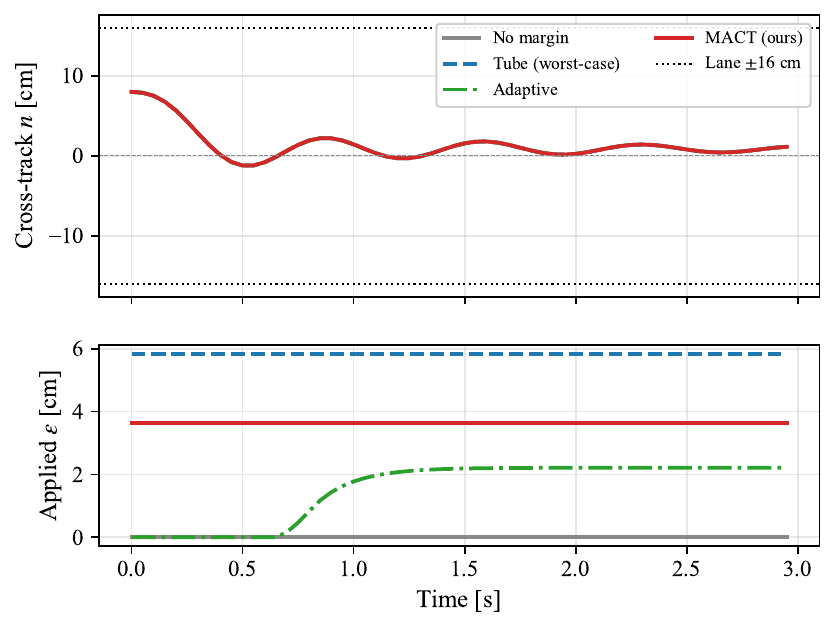}
\caption{Closed-loop trajectory at $v{=}15$\,m/s,
$\kappa{=}0.012$\,rad/m. Top: cross-track $n(t)$; all methods stay
within $\pm$LANE$_{hw}$. Bottom: applied $\varepsilon(t)$. Tube is
flat; adaptive ramps up slowly (misses the initial transient
window); MACT is at its correct scenario-specific level from the
first step.}
\label{fig:exp8b}
\end{figure}

\begin{figure}[tb]
\centering
\includegraphics[width=0.40\textwidth]{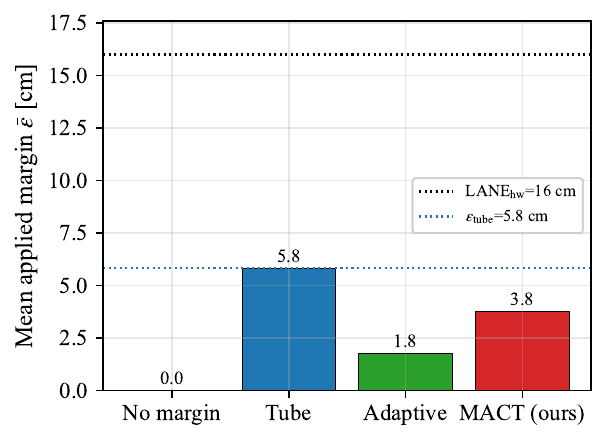}
\caption{Grid-averaged $\bar\varepsilon$. MACT saves
$34\%$ vs.\ tube while keeping the same safety.}
\label{fig:exp8c}
\end{figure}

\begin{table}[ht]
\centering
\caption{Closed-loop MPC (average over the $3{\times}3$ grid).
$\bar\varepsilon$ is mean applied tightening.}
\label{tab:exp8}
\begin{tabular}{lcccc}
\toprule
Method & Safe & $\bar\varepsilon$ & $\bar\varepsilon/\varepsilon_{\rm tube}$ & Solve \\
       & (\%) & (cm)              & ($\%$)                                   & (ms)  \\
\midrule
No margin          & 100 & 0.0 & 0   & 6.2 \\
Tube               & 100 & 5.8 & 100 & 6.6 \\
Adaptive           & 100 & 1.8 & 31  & 6.7 \\
\textbf{MACT (ours)} & \textbf{100} & \textbf{3.8} & \textbf{66} & 7.0 \\
\bottomrule
\end{tabular}
\end{table}

\addtolength{\textheight}{-3cm}

%%%%%%%%%%%%%%%%%%%%%%%%%%%%%%%%%%%%%%%%%%%%%%%%%%%%%%%%%%%%%%%%%%%%%%%%%%%%%%%%
\section{DISCUSSION}
\label{sec:disc}

\textbf{Relation to tube MPC.} MACT can be read as an \emph{adaptive
tube}: the tube size $\epsilon(v,\kappa)$ depends on the operating
state instead of being a constant as in classical tube
MPC~\cite{langson2004robust,mayne2005robust}. The dependence is
principled because the tube exactly follows the dominant mismatch
term $v^2\kappa$, which is the centripetal acceleration.

\textbf{The $v_c$ result.} The characteristic speed isolates a
directional transition in the initial response of the bicycle model.
Prior work on bicycle stability~\cite{sharma2016bicycle} characterizes
the self-balancing limit but not this initial-response transition.
The $v_c$ threshold immediately identifies the safety-critical regime
and gives the sign of the required margin.

\textbf{Relation to online calibration.} Online
calibration~\cite{wang2021calibration} reduces the mismatch itself,
while MACT gives the safety certificate around the calibrated model.
The analytical $a_2^{\mathrm{anal}}$ uses $v_c$ as a proxy for model
quality, so an online-updated $v_c$ directly shrinks the margin.
The same $T^2$-scaling idea may also generalize to contact-rich
planning in legged robotics~\cite{wang2018unified}, and to the
differential-dynamic-programming family~\cite{tassa2014control}, in
which similar mismatch-induced safety margins
arise~\cite{wang2017realtime,wang2018realization}.

\textbf{Generality (and the part we trust the least).} Experiment~7
is the result we find the most interesting and also the most
unfinished. The leaning bicycle reaches a curved trajectory by a
mechanism that has very little in common with how the 2-DOF car
does it: the lean angle has to first build up against gravity, and
only once the lean is established can the bicycle hold the curvature.
Despite that, the same $v^2\kappa$ envelope captures the deviation
($R^2{=}0.934$). We do not have a derivation of why the same scaling
should hold across the two mechanisms, and we are not claiming a
universal law; we are reporting what we observe.

\textbf{Limitations.} Three caveats are worth stating explicitly,
because the clean scaling law in Section~\ref{sec:mismatch} can
otherwise be over-read. (i) The whole derivation lives inside the
linear-tire regime ($a_{\mathrm{lat}}{\lesssim}0.4g$); once the front
axle approaches saturation the centripetal-acceleration deficit is no
longer the dominant mismatch term, and the bound flips from sufficient
to optimistic. We do not address that regime in this paper. (ii)
Cor.~\ref{cor:a2} is parametric in $v_{\max}$. If the realised vehicle
speed exceeds the assumed $v_{\max}$, the analytical guarantee silently
breaks; the cleaner engineering choice is to clip $v$ inside the
planner. (iii) Steering-actuator lag and post-settling-time yaw
oscillations are not part of the bound. Their effect is small in our
closed-loop experiments (worst per-step deviation $0.6$\,cm), but a
platform with sluggish steering would push more of the budget into a
pure actuator-lag margin that lives outside MACT and has to be added
separately. For the leaning bicycle the balance
controller~\cite{sharma2016bicycle,wang2014meng} is tuned at a single
speed; deriving $a_2$ directly from the lean time constant is left for
future work, as is online adaptation using real-time tracking data.

%%%%%%%%%%%%%%%%%%%%%%%%%%%%%%%%%%%%%%%%%%%%%%%%%%%%%%%%%%%%%%%%%%%%%%%%%%%%%%%%
\section{CONCLUSION}
\label{sec:conclusion}

There is one observation: above $v_c{=}\sqrt{C_\alpha L/M}$ the
kinematic-vs-dynamic mismatch is outward, and its peak deviation lies
inside a $v^2\kappa T^2$ envelope whose two asymptotic limits we can
write in closed form. The smallest sufficient coefficient,
$a_2^{\mathrm{anal}}{=}\tfrac{1}{2}(1{-}v_c^2/v_{\max}^2)T^2$, comes
from the vehicle parameters and the horizon alone --- no offline fit.
The 2-DOF sweeps confirm the scaling and show $84\%$ less wasted
margin than a fixed baseline at equal safety. The leaning bicycle
matches the same scaling for reasons we cannot yet derive
($R^2{=}0.934$). The closed-loop MPC trims $34\%$ off tube while
matching its safety, and beats an online-adaptive baseline on
worst-case responsiveness.

What remains open is the saturating regime, the actuator-lag
contribution, and the cross-platform agreement.

\bibliographystyle{IEEEtran}
\bibliography{main}

\end{document}